\documentclass{opt2023} 

\usepackage[numbers]{natbib}
\usepackage[utf8]{inputenc} 
\usepackage[T1]{fontenc}    
\usepackage{hyperref}       
\usepackage{url}            
\usepackage{booktabs}       
\usepackage{amsfonts}       
\usepackage{amsmath}       
\usepackage{amssymb}
\usepackage{bbm}
\usepackage{nicefrac}       
\usepackage{microtype}      
\usepackage{graphicx}
\usepackage{algorithm2e}

\usepackage[colorinlistoftodos,bordercolor=orange,backgroundcolor=orange!20,linecolor=orange,textsize=scriptsize]{todonotes}

\newtheorem{assumption}{Assumption}

\newtheorem*{theorem*}{Theorem}

\newcommand{\textinf}{\text{inf}}
\newcommand{\ellmax}{{\ell_\text{max}}}
\newcommand{\biased}{\text{biased}}
\newcommand{\naive}{\text{naive}}
\newcommand{\MLMC}{\text{MLMC}}
\newcommand{\DMLMC}{\text{DMLMC}}
\newcommand{\EE}{\mathbb{E}}

\newcommand{\RR}{\mathbb{R}}

\newcommand{\Ocal}{\mathcal{O}}

\newcommand{\rd}{\mathrm{d}}
\newcommand{\bsxi}{\boldsymbol{\xi}}

\DeclareMathOperator*{\minimize}{minimize}

\allowdisplaybreaks

\title[On the Parallel Complexity of MLMC in SGD]{On the Parallel Complexity of Multilevel Monte Carlo \\ in Stochastic Gradient Descent} 

\optauthor{%
    \Name{Kei Ishikawa} \Email{k.stoneriv@gmail.com}\\
    \addr Tokyo Institute of Technology, Japan \thanks{A major part of the work was done while the author was at ETH Zurich, Switzerland.}
}

\begin{document}

\maketitle


\begin{abstract}
In the stochastic gradient descent (SGD) for sequential simulations such as the neural stochastic differential equations, the Multilevel Monte Carlo (MLMC) method is known to offer better theoretical computational complexity compared to the naive Monte Carlo approach.
However, in practice, MLMC scales poorly on massively parallel computing platforms such as modern GPUs,  
because of its large parallel complexity which is equivalent to that of the naive Monte Carlo method.
To cope with this issue, we propose the delayed MLMC gradient estimator that drastically reduces the parallel complexity of MLMC by recycling previously computed gradient components from earlier steps of SGD. 
The proposed estimator provably reduces the average parallel complexity per iteration at the cost of a slightly worse per-iteration convergence rate.
In our numerical experiments, we use an example of deep hedging to demonstrate the superior parallel complexity of our method compared to the standard MLMC in SGD.
\end{abstract}

\section{Introduction}

In this paper, we study the stochastic gradient descent (SGD) for sequential stochastic simulations such as the neural stochastic differential equations (SDEs).
Since the seminal work by \citet{neural-odes} on neural ordinary differential equations (ODEs), the neural differential equations have gained considerable traction in the machine learning community.
The development of neural differential equations has led to various extensions beyond ODEs \citep{neural-odes} and SDEs \citep{nsde-basic},
such as neural jump ODEs \citep{krach2021neural} and SDEs \citep{jia2019neural}, neural control differential equations \citep{kidger2020neuralcde}, and neural stochastic partial differential equations (PDEs) \citep{salvi2022neural}.

In the Monte Carlo simulation of such sequential models, the multilevel Monte Carlo (MLMC) method \citep{giles2008multilevel, giles2015multilevel} is a popular choice to improve the computational complexity compared to the naive Monte Carlo approach.
Notably, recent research by \citet{ko2023multilevel} explored the application of MLMC to neural SDEs, and \citet{hu2021bias} conducted an in-depth theoretical analysis, shedding light on the performance improvements that MLMC offers to SGD. 
However, when it comes to practical implementation, MLMC exhibits a scalability issue when combined with SGD on massively parallel computers such as GPUs.

In highly concurrent settings, where we can increase the batch size significantly to reduce the impact of the stochasticity in the gradient estimator during optimization, the primary bottleneck for performance shifts from standard computational complexity to the parallel complexity\footnote{This is often called the work time of the algorithm on parallel random-access machine model \citep{jaja1992introduction}.} of the gradient estimator. 
In such a scenario, the limiting factor for performance becomes the total number of iterations in the SGD process, which is constrained by the parallel complexity of the gradient estimator, and consequently, the benefit of the variance reduction offered by MLMC becomes marginal.
This shift poses a challenge to the effective use of MLMC within the SGD framework on high-performance parallel hardware, as the traditional MLMC estimator requires the same parallel complexity as the naive Monte Carlo estimator.

To address this issue, we propose a novel adaptation of the traditional MLMC approach, which we term "delayed MLMC."
By periodically sampling the expensive parts of gradients and reusing those values for the rest of the time, the delayed MLMC achieves a substantial reduction in parallel complexity.
Later in the paper, we provide theoretical analysis, where we derive the convergence rate of the SGD with the delayed MLMC for a smooth non-convex objective, and demonstrate its practical relevance by a numerical experiment using a neural SDE model.

Before delving into the technical discussions, we offer a brief comparison of our work with existing literature on variance reduction techniques for SGD. 
Some of the most popular variance reduction techniques such as SAG \citep{schmidt2017minimizing}, SVRG \citep{johnson2013accelerating}, SAGA \citep{defazio2014saga}, and SPIDER \citep{fang2018spider} are orthogonal to our approach and they may be combined with our method.
Still, they share a similarity with our method in that they also take advantage of the smoothness of the loss function, which requires gradients for two similar parameters to be proportionally similar.
Our work aligns most closely with that of \citet{hu2021bias}, which examines the convergence of SGD with MLMC, albeit with a primary focus on standard complexity, whereas our focus is on parallel complexity.

\section{Background on Multilevel Monte Carlo Method}

Here, we provide an overview of the conventional MLMC estimator of the gradient that offers better computational complexity than the naive Monte Carlo gradient estimator.
Conceptually, MLMC reduces the variance of the Monte Carlo average by considering a hierarchy of the discretization size for a simulation.
Instead of running the most precise but very expensive random simulations many times and taking their average, it combines the results of a large number of cheap but low-accuracy simulations with a very small number of highly accurate simulations.

To formally discuss this concept, we introduce a sequence of approximations for the true random simulation, denoted as $\{\hat F_\ell(x, \xi)\}_{\ell=0}^{\ellmax}$
along with their expectations $\{F_\ell(x)\}_{\ell=0}^{\ellmax}$ so that $F_\ell(x)= \EE_\xi[\hat F_\ell(x, \xi)]$.
Here, $x\in\RR^m$ represents a parameter of the function and $\xi$ is a random variable.
The quality of the approximation improves as we increase level $\ell$ with the maximum level offering the best possible approximation.
Given such approximations, we would like to solve optimization problem
\begin{equation}
    \minimize_{x\in\RR^m} F(x) := F_\ellmax(x).
\end{equation}
For instance, in the case of SDEs, the approximation at level $\ell$ corresponds to the SDE simulation with step size $\Delta t = \delta 2 ^{-\ell}$.
Although it is possible to use the naive SGD with gradient estimator $\nabla \hat F_\ellmax(x, \xi)$\footnote{Here, derivative is taken with respect to the parameter $x$ so that $\nabla = \left(\frac{\rd}{\rd x_1},\frac{\rd}{\rd x_2}, \ldots,\frac{\rd}{\rd x_m}\right)^T$.} to optimize the above objective, for many types of sequential simulations such as SDEs, it is known that we can construct a more sample efficient gradient estimator using MLMC.

For MLMC to be applicable to the problem above, we need to make additional assumptions:

\begin{assumption}[Complexity of the Gradient]\label{assum:complexity}
Both standard and parallel complexities of estimator $\nabla \hat F_\ell$ grow exponentially to $\ell$.
In other words, there exists constant $c$ such that for any $x$ and $\ell$,
\begin{equation*}
\mathrm{Complexity}\left[\nabla \hat F_\ell(x, \xi)\right] = \Ocal(2^{c\ell}) \text{\ \ and \ \ }
\mathrm{ParallelComplexity}\left[\nabla \hat F_\ell(x, \xi)\right] = \Ocal(2^{c\ell}).
\end{equation*}
\end{assumption}

Furthermore, we introduce so-called coupled estimator $\Delta_\ell \hat F(x, \xi):=\hat F_\ell(x, \xi) - \hat F_{\ell - 1}(x, \xi)$ for $\ell=0, \ldots, \ellmax$.
Here, for notational simplicity, we set $\hat F_{-1}(x, \xi) = 0$.
This coupled estimator allows us to decompose original stochastic objective $\hat F_\ellmax(x, \xi)$ as 
$\hat F_\ellmax(x, \xi) = \sum_{\ell=0}^\ellmax \Delta_\ell F(x, \xi)$.
Similarly, we define the difference of the expectations as $\Delta_\ell F(x) := F_\ell(x) - F_{\ell-1}(x)$.
For this gradient estimator of difference $\nabla \Delta_\ell \hat F(x, \xi)$, we make the following assumption on the exponential decay of variance. 

\begin{assumption}[Decay of the Variance]\label{assum:variance}
There exist constants $M$ and $b$ such that for any $\ell$ and $x$,
$$
\EE_\xi\left\|\nabla \Delta_\ell \hat F(x, \xi) - \nabla \Delta_\ell F(x)\right\|^2
\leq 2^{-b\ell} M.
$$
Furthermore, we assume that the decay rate of variance, represented by the parameter $b$, is faster compared to the increase rate of the cost $c$ defined in Assumption \ref{assum:complexity}, so that $b > c$.
\footnote{
    The latter assumption is made for the sake of simplicity.
    Though it is not always required for MLMC to achieve better convergence than the naive Monte Carlo method, it is essential for the fastest convergence rate of MLMC \cite{giles2015multilevel}.
}
\end{assumption}

Here, to understand the feasibility of this assumption in practice, let us consider an example of an SDE simulation.
In the case of SDEs, the coupled estimator corresponds to the difference between two simulations using different discretizations of the same continuous Brownian motion path $\xi$. 
Compared to coarse simulation $\nabla \hat F_{\ell-1}(x, \xi)$ in the previous level, $\nabla \hat F_\ell(x, \xi)$ uses a finer time grid with half the step size.
As both of them approximate the same SDE solution given a Brownian path, the difference between them tends to decay quickly, as in the assumption.
\footnote{
    The above assumption holds for $b=2k$ if we use an SDE solver with strong order $k$ \citep{kloedenplaten1992} for computing $\hat F_\ell(x, \xi)$'s and its gradient by adjoint method \citep{scalable-sde}.
    Indeed, a common choice of an SDE solver for MLMC is the Milstein scheme \citep{giles2008improved}, and it has strong order $k=1$.
    Nevertheless, Assumption \ref{assum:variance} (and \ref{assum:smoothness}) cannot always be guaranteed theoretically. In such cases, one has to confirm this assumption experimentally, as we have done in our numerical experiment.
}

Now, we introduce the MLMC gradient estimator with effective batch size $N$ as
\begin{equation*}
   \nabla \hat F_\MLMC 
   = \sum_{\ell=0}^{\ellmax} \frac{1}{N_\ell} \sum_{n=1}^{N_\ell}
   \nabla \Delta_\ell \hat F(x, \xi_{\ell,n}),
\end{equation*}
where $N_\ell = \left\lceil \frac{2^{-(b + c)\ell / 2}}{\sum_{\ell=0}^\ellmax 2^{-(b + c)\ell / 2}}\cdot N\right\rceil = \Theta(N 2 ^{-(b + c) \ell / 2})$.
\footnote{
    Here, $\lceil \cdot \rceil$ and $\lfloor \cdot \rfloor$ denote the ceiling function and the floor function, respectively.
}
This estimator has $\sum_{\ell=0}^\ellmax N_\ell \Ocal\left(2^{c\ell}\right) = \Ocal(N)$ complexity and
$\sum_{\ell=0}^\ellmax \frac{M 2^{-b\ell}}{N_\ell} = \Ocal(N^{-1})$ variance due to the assumption that $b > c$.
Thus, the MLMC estimator is more efficient than the naive Monte Carlo estimator 
\begin{equation*}
   \nabla \hat F_\naive = \frac{1}{N}\sum_{n=1}^N\nabla \hat F_\ellmax(x, \xi_n),
\end{equation*}
with $\Ocal(N 2^{c\ellmax})$ complexity and $\Ocal(N^{-1})$ variance.
In Appendix \ref{app:mlmc}, we describe the derivation of this optimal allocation of per-level sample size $N_\ell$ for MLMC that we used here.

\section{Delayed Multilevel Monte Carlo method for SGD}

As discussed above the MLMC estimator has superior computational complexity than the naive Monte Carlo estimator. 
However, the computation of the MLMC gradient always requires the computation of the highest level with $\Ocal(2^\ellmax)$ parallel complexity, which makes the MLMC-based SGD as slow as the naive SGD on a massively parallel computer.

To cope with this problem, we propose the delayed MLMC for SGD, which we describe in Algorithm \ref{algo:delayed_mlmc}, where we introduced
$\nabla \Delta_\ell \hat F_\MLMC(x, \bsxi_{\ell})
= \frac{1}{N_\ell} \sum_{n=1}^{N_\ell} \nabla \Delta_\ell \hat F(x, \xi_{\ell,n})$ 
for $\bsxi_\ell = \left(\xi_{\ell, 1}, \ldots, \xi_{\ell, N_\ell}\right)$.
With this notation, the standard MLMC estimator at step $t$ can be written as 
$\nabla \hat F_\MLMC^{(t)}  = \sum_{\ell=0}^\ellmax\nabla \Delta_\ell \hat F_\MLMC(x_t, \bsxi_{t, \ell})$
whereas the delayed MLMC estimator becomes
\begin{equation*}
   \nabla \hat F_\DMLMC^{(t)}
   = \sum_{\ell=0}^{\ellmax} \frac{1}{N_\ell} \sum_{n=1}^{N_\ell}
   \nabla \Delta_\ell \hat F_\MLMC(x_{\tau_\ell(t)}, \bsxi_{\tau_\ell(t), \ell}).
\end{equation*}
Instead of calculating the gradient at each level every time step, the delayed MLMC estimator computes the gradient at level $\ell$ only once per every $\lfloor 2^{d\ell} \rfloor$ steps, and when the gradient computation is skipped, it reuses the most recent gradient computed at time $\tau_\ell(t)$, which satisfies $t - \lfloor 2^{d\ell} \rfloor \leq \tau_\ell(t) \leq t$ and 
$\tau_\ell(t) \equiv 0 \mod \lfloor2^{d\ell}\rfloor$.
Under Assumption \ref{assum:complexity}, the parallel complexities of the standard SGD and the MLMC-based SGD per iteration are both $\Ocal\left( 2^{c\ellmax} \right)$.
In contrast, the average parallel complexity of the delayed MLMC gradient descent (Algorithm \ref{algo:delayed_mlmc}) per iteration is 
$\Ocal\left(\sum_{\ell=0}^{\ellmax} 2^{(c - d)\ell} \right)$, which is an improvement by a factor of $2^{d\ellmax}$ to $2^{c\ellmax}$, depending on the magnitude of $c$ and $d$.
\footnote{
    Summation $\sum_{\ell=0}^{\ellmax} 2^{(c - d)\ell}$ becomes 
    $\Ocal(1)$ for $c < d$,
    $\Ocal(\ellmax)$ for $c = d$, and
    $\Ocal(2^{(c-d)\ellmax})$ for $c > d$.
}
Here, a natural question to ask is to how much extent we can tolerate the bias in the gradient introduced by the delayed scheme.
In the next section, we answer this question with theoretical analysis by showing that it can be controlled to negligible magnitude when learning rate $\alpha_t$ is taken small enough.

\begin{algorithm2e}[htbp]
\caption{SGD with the delayed MLMC}\label{algo:delayed_mlmc}
\SetAlgoLined
Initialize $x_0$.\\
\For{$t = 0, \ldots, T$}{
    \For{$\ell = 0, \ldots, \ellmax$}{
        \If{$t \equiv 0 \mod \lfloor2^{d\ell}\rfloor$}{
            Sample a new gradient at current $x_t$ as 
            $\nabla \Delta_\ell \hat F^{(t)} \gets \nabla \Delta_\ell \hat F_\MLMC(x_t, \bsxi_{t, \ell})$\footnotemark.
            
            Update the time of the latest gradient as $ \tau_\ell \gets t$.
        }
    }
    Compute delayed gradient estimator as 
   $\nabla \hat F^{(t)}_\DMLMC \gets \sum_{\ell=0}^{\ellmax} \nabla \Delta_\ell \hat F^{(\tau_\ell)}$.
   
    Update the parameter $x$ as
    $x_{t+1} \gets x_t - \alpha_t \nabla \hat F^{(t)}_\DMLMC$.
}
\end{algorithm2e}

\footnotetext{
Here, $\bsxi_{t, \ell}$ are sampled independently from past samples and samples from the other levels.
}

\section{Theoretical Guarantee}

In this section, we present the results of our theoretical analysis of the delayed MLMC. To justify the skipping of gradient computation, we make the following assumption regarding smoothness.

\begin{assumption}[Decay of the Smoothness]\label{assum:smoothness}
There exist constants $L$ and $d$ such that for any $x_1, x_2$ and $\ell$,
$$\left\|\nabla \Delta_\ell F(x_1) - \nabla \Delta_\ell F(x_2) \right\| \leq 2^{-d\ell} L \|x_1 - x_2\|.$$
\end{assumption}

This assumption guarantees that gradients at higher levels undergo progressively smaller changes throughout the optimization process, enabling us to skip the computation of higher levels in the delayed MLMC.
Here, note that we can trivially obtain the standard smoothness condition for SGD  from this assumption as 
$\left\| \nabla F(x_1) - \nabla F(x_2) \right\|
\leq  \sum_{\ell=0}^{\ellmax} \left\| \nabla \Delta_\ell F(x_1) - \nabla \Delta_\ell F(x_2) \right\| 
= \sum_{\ell=0}^{\ellmax} 2^{-d\ell} L \|x_1 - x_2\|$.
Thus, $\nabla F(x)$ is $L'$-smooth for $L':= \left( \sum_{\ell=0}^\infty 2^{-d\ell} \right) L$.
With this additional assumption in place, we can now present our main theorem (with the proof available in the appendix):
\begin{theorem}[Delayed MLMC Gradient Descent for Non-Convex Functions]\label{thm:delayed_mlmc}
Under Assumption \ref{assum:variance} and \ref{assum:smoothness}, suppose we run SGD with delayed MLMC gradient estimator $\nabla \hat F^{(t)}_\DMLMC$ as in Algorithm \ref{algo:delayed_mlmc}.
Assume that the step sizes are chosen as $\alpha_t=\alpha_0\leq\min\left\{\frac{1}{8L'}, \frac{\beta}{L}\right\}$ for $\beta$ satisfying $0 < \beta \leq \frac{1}{12(\ellmax + 1)\left(\sum_{\ell=0}^{\infty}2^{-d\ell}\right)\log(2T+1)}$. Then, we have
\begin{equation*}
    \frac{1}{T}\sum_{t=0}^{T-1}\EE\|\nabla F(x_t)\|^2
    \leq \frac{8(F(x_0) - F_\textinf)}{\alpha_0 T}
    + \left( 24\ellmax + \frac{49}{2} \right) M'
    \leq \Ocal\left(\left(\frac{\log T}{T} + \frac{M}{N}\right)\ellmax\right),
\end{equation*}
where we defined 
$F_\textinf:= \inf_{x}F(x)$ and 
$M' :=
\frac{M}{N}
    \left(\sum_{\ell=0}^\ellmax 2^{-(b + c)\ell / 2}\right)
    \left(\sum_{\ell=0}^\ellmax 2^{-(b - c)\ell / 2}\right)
$.
\end{theorem}

As can be seen, our convergence rate depends on variance of the gradient $\frac{M}{N}$, but with a massively parallel computer, we can take the (effective) batch size $N$ sufficiently large to reduce the variance term to a negligible magnitude.
Then, the convergence rate of delayed MLMC becomes $\Ocal\left(\left(\frac{\log T}{T}\right)\ellmax\right)$, which is slightly less favorable than $\Ocal(\frac{1}{T})$ rate of both MLMC and naive method.
At the cost of the additional factor of $\Ocal\left(\log T \cdot \ellmax\right)$, the delayed MLMC gains substantial improvement in its parallel complexity as discussed earlier.
For clarity, in Table \ref{tab:comparison}, we provide a summarized comparison of the convergence rate and the complexities of these methods.

\begin{table}[h]
    \centering
    \begin{tabular}{|c|c|c|c|}
        \hline
             & Convergence rate & Complexity & Parallel complexity \\
        \hline
         Naive SGD & $\Ocal(\frac{1}{T} + \frac{M}{N} \ellmax)$ & $\Ocal(NT 2^{c\ellmax})$ & $\Ocal(T 2^{c\ellmax})$ \\
         MLMC + SGD & $\Ocal(\frac{1}{T} + \frac{M}{N})$ & $\Ocal(NT)$ & $\Ocal(T 2^{c\ellmax})$ \\
         Delayed MLMC + SGD (ours) & $\Ocal((\frac{\log T}{T} + \frac{M}{N}) \ellmax)$& $\Ocal(NT)$ &  $\Ocal(T\sum_{\ell=0}^\ellmax 2^{(c-d)\ell})$\footnotemark \\
         \hline
    \end{tabular}
    \caption{
        A comparison of the convergence rate of $\frac{1}{T}\sum_{t=0}^{T-1}\EE\|\nabla F(x_t)\|^2$, the (standard) complexity, and the parallel complexity of different methods.
        Parameters $T$, $M$, and $N$ are the number of iterations in SGD, the variance of the gradient, and the effective batch size, respectively.
        As $\frac{M}{N}$ can be ignored if a massively parallel computer is available to increase the batch size $N$ arbitrarily large, the leading terms become dominant in the convergence rate.
    }
    \label{tab:comparison}
\end{table}

\footnotetext{
    Again, this summation becomes 
    $\Ocal(T)$ for $c < d$,
    $\Ocal(T\ellmax)$ for $c = d$, and
    $\Ocal(T2^{(c-d)\ellmax})$ for $c > d$.
}

\section{Experiments}

In the numerical experiments, we employed an example of deep hedging \citep{buehler2019deep} to assess the performance of the proposed method. 
In the context of deep hedging, our goal is to solve optimization problem
$\min_{\theta\in\Theta, p_0\in\RR}\EE\left| \max\{S_1 - K, 0\} - \int_0^1 H_\theta(t, S_t)\rd S_t - p_0 \right|^2$, to find the optimal hedging strategy $H(t, s)$.
In our experiments, we chose the underlying asset price process $S_t$ to be a geometric Brownian motion, and hedging strategy $H_\theta(t, s)$ was parameterized using a deep neural network. For more detailed information regarding the experimental setup, please refer to Appendix \ref{app:experimental_settings}.

\begin{figure}[htbp]
    \vspace{-1em}
    \centering
    \includegraphics[width=70mm, trim={15 10 40 10mm},clip]{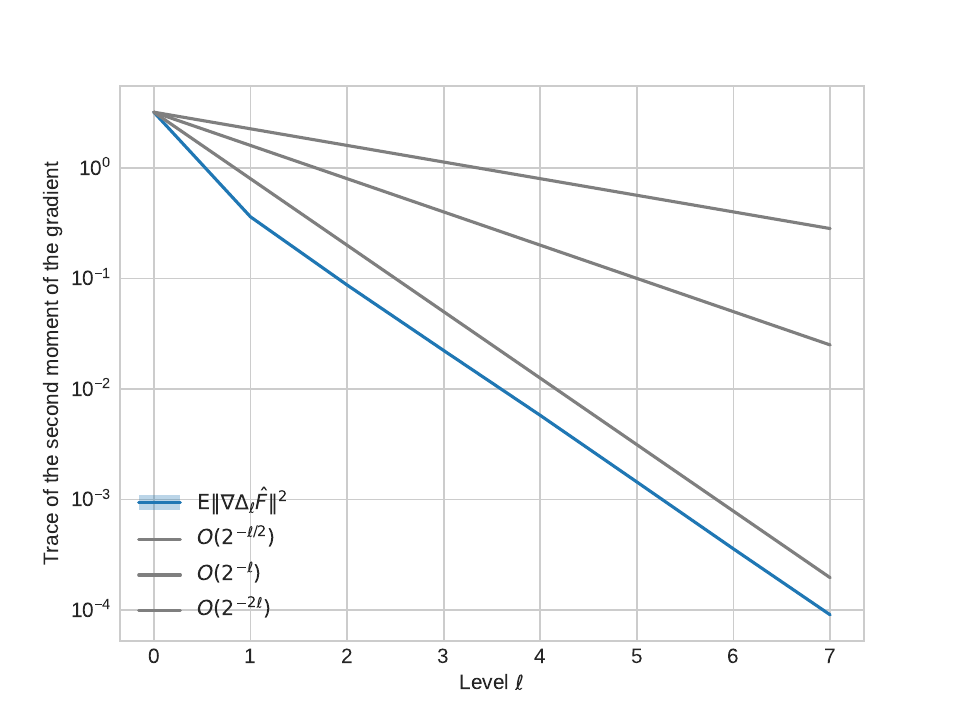}
    \hfill
    \includegraphics[width=70mm, trim={15 10 40 10mm},clip]{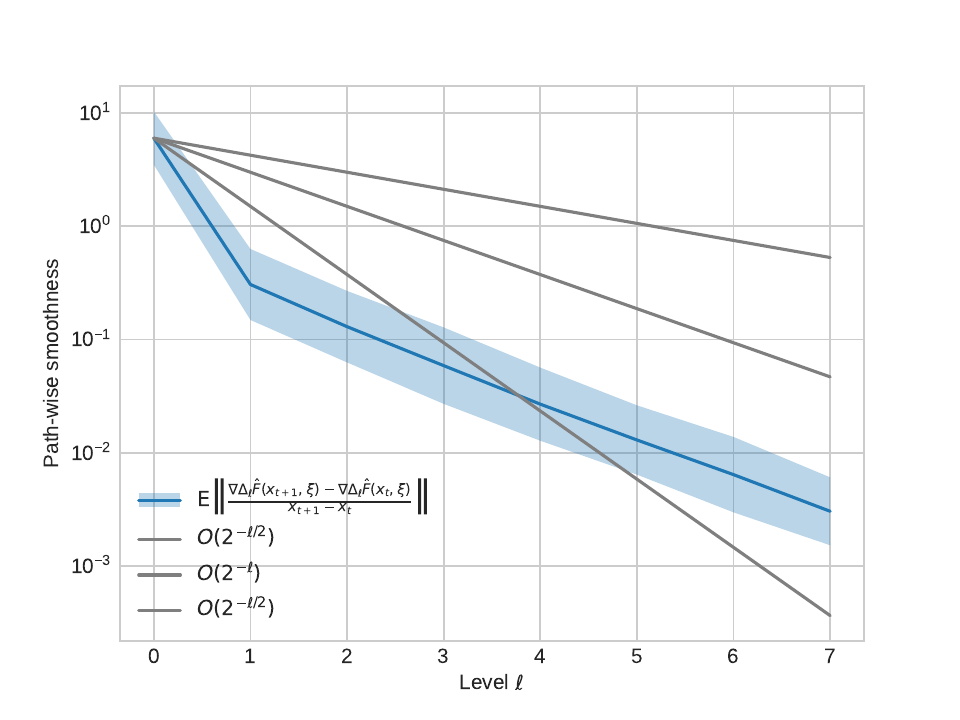}
    \vspace{-0.5em}
    \caption{
        The decay of squared norm of the gradient 
        $\EE\left\|\nabla\Delta_\ell \hat F(x, \xi) \right\|^2$ (left)
        and path-wise smoothness 
        $\EE\left\|\frac{\nabla\Delta_\ell \hat F(x_{t+1}, \xi) - \nabla\Delta_\ell \hat F(x_t, \xi)}{x_{t+1} - x_t}\right\|$ (right).
        These expectations were estimated by running the Monte Carlo simulation for the parameters during the optimization. The line and the band indicate the mean and standard deviation of the estimated values.
    }
    \label{fig:decay}
    \vspace{-1.5em}
\end{figure}

To assess the validity of Assumption \ref{assum:variance} and \ref{assum:smoothness}, we examined the decay rate of the variance and the smoothness of $\nabla \Delta_\ell \hat F(x, \xi)$ during the optimization, as shown in Figure \ref{fig:decay}.
To estimate the decay rate of the variance, we instead tracked squared norm of the gradient $\EE\left\| \nabla \Delta_\ell \hat F(x, \xi_{t, \ell})\right\|^2$, which provides an upper bound on the variance.
Since the direct estimation of the smoothness is difficult, we approximated it with the path-wise smoothness in $L^1$ norm.
From these figures, we can reasonably assume the values of $b$ to be close to $2$ and $d$ to be $1$ in the aforementioned assumptions, which implies that the standard MLMC and the delayed MLMC are applicable to our problem.

\begin{figure}[htbp]
    \centering
    \includegraphics[width=70mm, trim={15 10 40 10mm},clip]{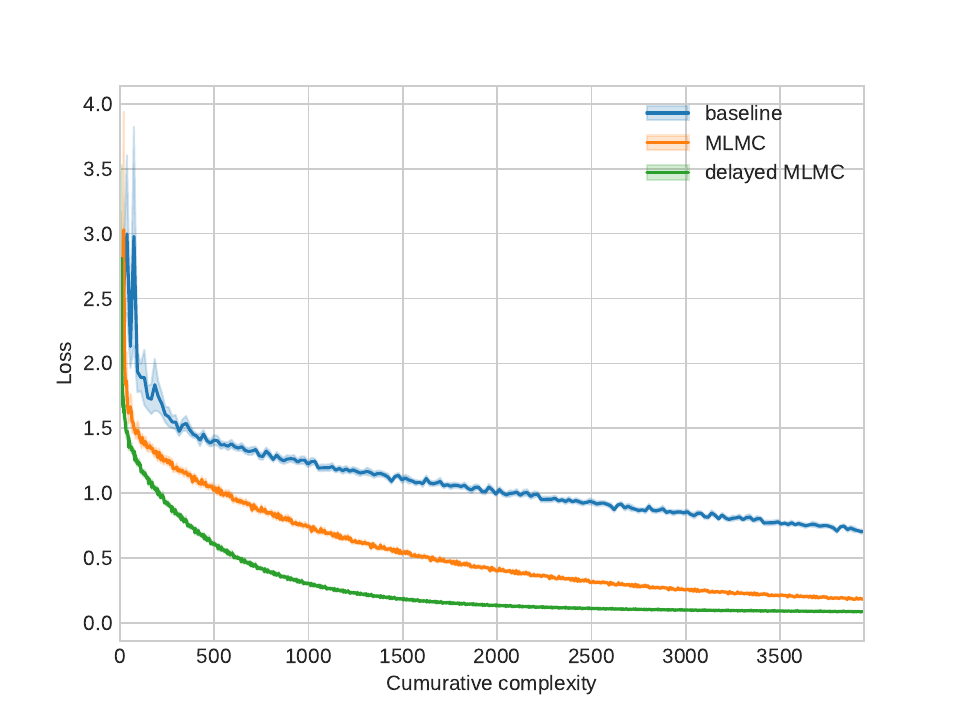}
    \hfill
    \includegraphics[width=70mm, trim={15 10 40 10mm},clip]{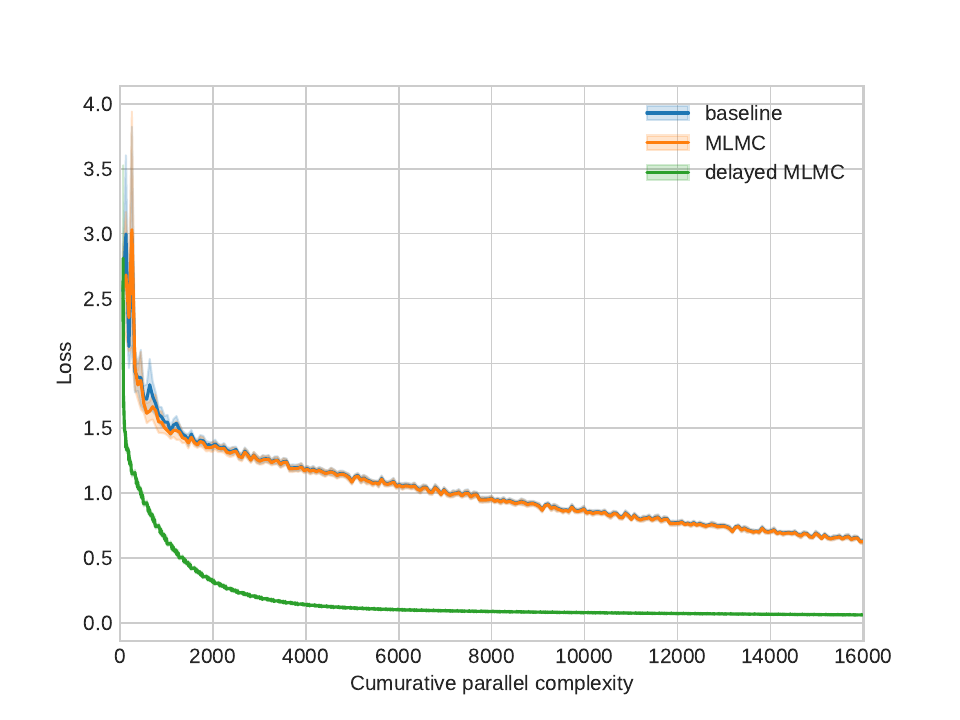}
    \caption{
    The mean and standard deviation of the loss throughout the optimization process for the naive SGD (baseline), the SGD with MLMC, and the SGD with delayed MLMC.
    Standard complexity (left) and parallel complexity (right) are used as the time scale of the learning curves.
    The mean and standard deviation are taken over 10 different runs.
    }
    \label{fig:learning_curve}
    \vspace{-1.5em}
\end{figure}

Figure \ref{fig:learning_curve} illustrates the learning curves for three different optimization methods: the naive SGD (baseline), the SGD with the standard MLMC, and the SGD with delayed MLMC.
All methods employed the same learning rate, and the batch sizes were adjusted to match the gradient variance across methods.
When we consider parallel complexity as the horizontal axis, it becomes evident that the delayed MLMC outperforms both the baseline and the standard MLMC, which aligns with our expectations.
Interestingly, even when assessing performance in terms of standard complexity as the time scale, we observe that the delayed MLMC exhibits slightly faster optimization compared to the standard MLMC.
This improvement can be attributed to the skipped computation of gradients at higher levels, which again demonstrates the effectiveness of the proposed approach.

\section{Conclusions}
We introduced the delayed MLMC gradient estimator to address scalability challenges in MLMC for SGD on massively parallel computers. Our method demonstrates significant parallel complexity improvements in both theory and experiments compared to standard MLMC in SGD. This contribution holds great potential for enhancing the scalability of neural differential equations and related simulations in the fields of machine learning and computational science.


\section*{Acknowledgement}
This work was greatly inspired by a research project the author conducted under the supervision by Florian Krach and Calypso Herrera at ETH Zurich.
The author also appreciates Takashi Goda at the University of Tokyo for technical advice on the Monte Carlo methods.

\bibliography{ref.bib}

\newpage

\appendix

\section{A Review on the Multilevel Monte Carlo Method}\label{app:mlmc}

\subsection{Literatures on MLMC}
The multilevel Monte Carlo (MLMC) method, initially introduced by \citet{heinrich2001multilevel} for parametric integration, gained substantial recognition following the seminal work by \citet{giles2008multilevel} on path simulation of SDEs.
Its applications have since extended into various domains, including partial differential equations with random coefficients \citep{cliffe2011multilevel}, continuous-time Markov chains \citep{anderson2012multilevel}, and nested simulations \cite{bujok2015multilevel}.
Following its success in the numerical simulations, the MLMC has been extended further to the fields of statistics and machine learning.
In statistics, MLMC has been applied to Markov chain Monte Carlo sampling \citep{dodwell2015hierarchical}, sequential Monte Carlo sampling \citep{beskos2017multilevel}, particle filtering \citep{jasra2017multilevel}. 
The adoption of MLMC in machine learning is a more recent development, with applications ranging from distributionally robust optimization \citep{levy2020large} to Bayesian computation \citep{ishikawa2021efficient, shi2021multilevel, li2023multilevel, chada2022multilevel} and reinforcement learning \citep{dixit2022multilevel, hoerger2023multilevel, zhang2023latent}.
For those interested in a comprehensive review of MLMC, \citet{giles2015multilevel} provides an excellent tutorial and an extensive survey.

\subsection{Optimal Sample Size Allocation of MLMC}
To determine the optimal sample sizes per level, denoted as $N_0, \ldots, N_\ellmax$, we minimize the variance under a fixed total cost.
\footnote{
    Here, we can alternatively minimize the total cost given a constant variance.
}
For the analysis, let us assume that there exist $C, c, M$, and $b$ such that
\begin{align*}
    C_\ell := \mathrm{Complexity}\left[\nabla \Delta_\ell \hat F(x, \xi)\right] = C 2 ^{c\ell}
\end{align*}
and that
\begin{align*}
    V_\ell := \EE\left\|\nabla \Delta_\ell \hat F(x, \xi) - \EE\nabla \Delta_\ell \hat F(x, \xi) \right\|^2 = M 2 ^{-b\ell}.
\end{align*}
Since the variance of the MLMC estimator
$\nabla \hat F_\MLMC  = \sum_{\ell=0}^{\ellmax} \frac{1}{N_\ell} \sum_{n=1}^{N_\ell} \nabla \Delta_\ell \hat F_\ellmax(x, \xi_{\ell,n})$
can be written as $\sum_{\ell=0}^\ellmax \frac{V_\ell}{N_\ell}$,
we can solve the following constrained optimization to obtain the optimal choice of $N_0, \ldots, N_\ellmax$.
\begin{equation*}
    \min_{N_0, \ldots, N_\ellmax > 0} \sum_{\ell=0}^\ellmax \frac{V_\ell}{N_\ell} \text{\ \ \ \ subject to\ \ \ \ }\sum_{\ell=0}^\ellmax C_\ell N_\ell = C_\text{total}.
\end{equation*}
This problem can be solved analytically by using the method of Lagrangian multipliers, yielding the optimal sample sizes: 
\begin{equation*}
    N_\ell = \frac{\sqrt{V_\ell/ C_\ell}}{\sum_{k=0}^\ellmax\sqrt{V_k C_k}} \cdot C_\text{total}
    \propto \sqrt{\frac{V_\ell}{C_\ell}}
    = \Ocal(2^{-(b + c)\ell / 2}).
\end{equation*}
For this solution, the resulting optimal variance is
\begin{equation*}
    \sum_{\ell=0}^\ellmax \frac{V_\ell}{N_\ell} 
    = \frac{1}{C_\text{total}} \cdot \left(\sum_{k=0}^\ellmax\sqrt{V_k C_k}\right)^2
    = \frac{1}{C_\text{total}} \cdot \left(\sum_{k=0}^\ellmax \Ocal(2^{-(b-c)\ell/2})\right)^2
    = \Ocal\left(\frac{1}{C_\text{total}}\right).
\end{equation*}

\section{Theoretical Analysis}\label{app:theoretical_analysis}

Here, we provide a convergence analysis of our algorithm for smooth and non-convex objective $F(x)$. 
To set the stage for the analysis of the delayed MLMC, we first present a well-established result for SGD applied to smooth and non-convex objectives as a reference.

\begin{theorem}[Stochastic Gradient Descent for Non-Convex Functions \citep{bottou2018optimization}]
\label{th:sgd}
Under Assumption \ref{assum:variance} and \ref{assum:smoothness},
suppose we run standard SGD with gradient estimator $\nabla \hat F(x_t, \xi_t)$ and step size $\alpha_t=\alpha_0\leq\frac{1}{L'}$ so that $x_{t+1} \gets x_t - \alpha_t \nabla \hat F(x_t, \xi_t)$ for i.i.d. samples of $\{\xi_t\}_{t=0}^{T-1}$.
Assume the variance of the gradient is bounded so that there exists constant $M_{\nabla \hat F}$ such that
$\EE_{\xi}\|\nabla \hat F(x, \xi) - \nabla F(x)\|^2 \leq M_{\nabla \hat F}$ for any $x$.
Then, we have
\begin{equation*}\label{eq:sgd_convergence}
    \sum_{t=0}^{T-1}\alpha_t\left(1 - \frac{L'}{2}\alpha_t\right)\EE\|\nabla F(x_t)\|^2
    \leq 
    F(x_0) - F_\textinf + \sum_{t=0}^{T-1}\frac{L' \alpha_t^2}{2}M_{\nabla \hat F}.
\end{equation*}
where we introduced $F_\textinf := \inf F(x)$.
\end{theorem}

\begin{proof}
This proof follows \citet{bottou2018optimization}, theorem 4.10.

Also, by Assumption \ref{assum:smoothness}, we know that $F(x)$ is $L'$-smooth so that 
\begin{align*}
F(x_{t+1}) 
&\leq  F(x_t) + \langle \nabla F(x_t), x_{t+1} - x_t\rangle + \frac{L'}{2} \|x_{t+1}-x_t\|^2 \\
&= F(x_t) + \langle \nabla F(x_t), -\alpha_t \nabla \hat F(x_t, \xi_t)\rangle + \frac{L'}{2} \|-\alpha_t\nabla \hat F(x_t, \xi_t)\|^2.
\end{align*}
By taking the expectation with respect to the stochasticity at time $t$ conditioned on the trajectory up to time $t-1$ (i.e. taking expectation with respect to $\xi_t$) and using Assumption \ref{assum:variance}, we get
\begin{align*}
\EE_{\xi_t} F(x_{t+1}) 
&\leq F(x_t) + \langle \nabla F(x_t), -\alpha_t \nabla F(x_t)\rangle + \frac{L'\alpha_t^2}{2} \left\{\|\nabla F(x_t)\|^2 + \EE_{\xi_t}\|\nabla \hat F(x_t, \xi_t) - \nabla F(x_t)\|^2\right\} \\
&\leq F(x_t) - \alpha_t \left(1 - \frac{L'}{2}\alpha_t\right)\|\nabla F(x_t)\|^2 + \frac{L'\alpha_t^2}{2} M_{\nabla \hat F}.
\end{align*}
By taking the summation of $F(x_t) - \EE_{\xi_t}F(x_{t+1}) + \frac{L'\alpha_t^2}{2} M_{\nabla \hat F}$ for $t=0, \ldots, T-1$ and taking the (non-conditional) expectation, we get
\begin{align*}
    \sum_{t=0}^{T-1}\alpha_t\left(1 - \frac{L'}{2}\alpha_t\right)\EE\|\nabla F(x_t)\|^2
    &\leq F(x_0) - \EE F(x_T) + \sum_{t=0}^{T-1}\frac{\alpha_t^2 L'}{2}M_{\nabla \hat F} \\
    &\leq F(x_0) - F_\textinf + \sum_{t=0}^{T-1}\frac{\alpha_t^2 L'}{2}M_{\nabla \hat F}.
\end{align*}
\end{proof}

\begin{remark}
Here, we can substitute $M_{\nabla \hat F}$ with the upper bound on the variance of $\nabla \hat F_\naive$ and $\nabla \hat F_\MLMC$ to obtain the convergence rate in Table \ref{tab:comparison}.
Under Assumption \ref{assum:variance}, the variance of the naive Monte Carlo estimator can be bounded as
\begin{align*}
\EE_{\xi}\|\nabla \hat F_\naive(x) - \nabla F(x)\|^2 
&\leq \frac{1}{N} \EE_{\xi}\|\nabla \hat F_\ellmax(x, \xi) - \nabla F_\ellmax(x)\|^2  \\
&\leq \frac{\ellmax + 1}{N} \sum_{\ell=0}^\ellmax \EE_{\xi}\|\nabla \Delta_\ell \hat F(x, \xi) - \nabla \Delta_\ell F(x)\|^2 \\
&\leq \frac{\ellmax + 1}{N} \sum_{\ell=0}^\ellmax \EE_{\xi}\|\nabla \Delta_\ell \hat F(x, \xi)\|^2 \\
&\leq \frac{(\ellmax + 1)M}{N} \sum_{\ell=0}^\ellmax 2^{-b\ell} \\
&\leq \Ocal\left(\frac{M}{N}\ellmax \right).
\end{align*}
Similarly, the variance of the standard MLMC estimator can be bounded as 
\begin{align*}
\EE_{\xi}\|\nabla \hat F_\MLMC(x) - \nabla F(x)\|^2 
&= \sum_{\ell=0}^\ellmax \EE_{\xi}\|\nabla \Delta_\ell \hat F_\MLMC(x) - \nabla \Delta_\ell F(x)\|^2 \\
&= \sum_{\ell=0}^\ellmax \frac{1}{N_\ell}\EE_{\xi}\|\nabla \Delta_\ell \hat F(x) - \nabla \Delta_\ell F(x)\|^2 \\
&= \sum_{\ell=0}^\ellmax \left\lceil \frac{2^{-(b + c)\ell / 2}}{\sum_{\ell=0}^\ellmax 2^{-(b + c)\ell / 2}}\cdot N\right\rceil^{-1} \cdot 2 ^{-b\ell} M \\
&\leq \frac{M}{N}
    \left(\sum_{\ell=0}^\ellmax 2^{-(b + c)\ell / 2}\right)
    \left(\sum_{\ell=0}^\ellmax 2^{-(b - c)\ell / 2}\right) \\
&=\Ocal\left(\frac{M}{N}\right),
\end{align*}
where we used the mutual independence of coupled estimators at different levels at the first line.
For notational convenience, we let
$M' :=
\frac{M}{N}
    \left(\sum_{\ell=0}^\ellmax 2^{-(b + c)\ell / 2}\right)
    \left(\sum_{\ell=0}^\ellmax 2^{-(b - c)\ell / 2}\right)
= M_{\nabla \hat F_\MLMC}
$
for the rest of the paper to represent the upper bound on the variance of the MLMC gradient estimator.
\end{remark}

Now, to study the convergence property of the delayed MLMC, we analyze the convergence of SGD with a general biased gradient estimator.

\begin{lemma}[Biased Stochastic Gradient Descent for Non-Convex Functions]\label{th:biased_sgd}
Under Assumption \ref{assum:variance} and \ref{assum:smoothness}, suppose we run SGD with a biased gradient estimator $\nabla \hat F^{(t)}_\biased = \nabla \hat F_\biased (x_t, \xi_{0:t}) $ depending on the information up to time $t$.
Assume the variance of the biased gradient is bounded so that there exists constant $M_{\nabla \hat F_\biased}$ such that
$\EE_{\xi_t}\|\nabla \hat F^{(t)}_\biased - \EE_{\xi_t}\hat F^{(t)}_\biased \|^2 \leq M_{\nabla \hat F_\biased}$ holds for any $x_0, t$, and $\xi_0, \ldots, \xi_{t-1}$.
Then, we have
\begin{align*}
    &\sum_{t=0}^{T-1}\alpha_t\left(\frac{1}{2} - L'\alpha_t\right)\EE\|\nabla F(x_t)\|^2 \\
    &\leq F(x_0) - F_\textinf + \sum_{t=0}^{T-1} \left[
    \alpha_t \left(\frac{1}{2} + L'\alpha_t\right) \EE \| \nabla F(x_t) 
    - \nabla \hat F^{(t)}_\biased \|^2
    + \frac{L'M_{\nabla \hat F_\biased}\alpha_t^2}{2}
    \right].
\end{align*}
\end{lemma}

\begin{proof}
We can use a similar argument to Theorem \ref{th:sgd} and get
\begin{align*}
    &\EE_{\xi_t} F(x_{t+1}) \\
    &\leq F(x_t) -\alpha_t \langle \nabla F(x_t), \EE_{\xi_t}\nabla \hat F^{(t)}_\biased\rangle + \frac{L'}{2} \EE_{\xi_t}\|-\alpha_t\nabla \hat F^{(t)}_\biased\|^2 
    \\
    &= F(x_t) -\alpha_t \| \nabla F(x_t)\|^2 
    + \alpha_t \langle \nabla F(x_t), \nabla F(x_t) - \EE_{\xi_t}\nabla \hat F^{(t)}_\biased \rangle \\
    &\quad \quad \quad \quad + \frac{L'\alpha_t^2}{2}
    \left\{\|\EE_{\xi_t}\nabla \hat F^{(t)}_\biased\|^2  + \EE_{\xi_t}\|\nabla \hat F^{(t)}_\biased - \EE_{\xi_t}\nabla \hat F^{(t)}_\biased \|^2\right\} 
    \\
    &\stackrel{2\langle v, u\rangle \leq \|v\|^2 + \|u\|^2}{\leq} F(x_t) - \alpha_t \left(1 - \frac{L'}{2}\alpha_t\right) \| \nabla F(x_t)\|^2 
    + \frac{\alpha_t}{2}\left\{
    \|\nabla F(x_t)\|^2 + \| \nabla F(x_t) - \EE_{\xi_t}\nabla \hat F^{(t)}_\biased \|^2 
    \right\} \\
    &\quad \quad \quad \quad + \frac{L'\alpha_t^2}{2}
    \left\{\|\EE_{\xi_t}\nabla \hat F_\biased^{(t)}\|^2  - \|\nabla F(x_t)\|^2 
    + M_{\nabla \hat F_\biased} \right\} 
    \\
    &= F(x_t) - \alpha_t \left(\frac{1}{2} - \frac{L'}{2}\alpha_t\right) \| \nabla F(x_t)\|^2 
    + \frac{\alpha_t}{2} \| \nabla F(x_t) - \EE_{\xi_t}\nabla \hat F^{(t)}_\biased \|^2 \\
    &\quad \quad \quad \quad + \frac{L'\alpha_t^2}{2}
    \left\{
    \left\langle \nabla F(x_t) + \EE_{\xi_t}\nabla \hat F^{(t)}_\biased, 
    \nabla F(x_t) - \EE_{\xi_t}\nabla \hat F^{(t)}_\biased \right\rangle
    + M_{\nabla \hat F_\biased} \right\} 
    \\
    &\stackrel{\text{Cauchy-Schwarz}}{\leq} F(x_t) - \alpha_t \left(\frac{1}{2} - \frac{L'}{2}\alpha_t\right) \| \nabla F(x_t)\|^2 
    + \frac{\alpha_t}{2} \| \nabla F(x_t) - \EE_{\xi_t}\nabla \hat F^{(t)}_\biased \|^2 \\
    &\quad \quad \quad \quad + \frac{L'\alpha_t^2}{2}
    \left\{
    \|\nabla F(x_t) + \EE_{\xi_t}\nabla \hat F^{(t)}_\biased\| \cdot
    \|\nabla F(x_t) - \EE_{\xi_t}\nabla \hat F^{(t)}_\biased\|
    + M_{\nabla \hat F_\biased} \right\}
    \\
    &\stackrel{\text{triangular ineq.}}{\leq} F(x_t) - \alpha_t \left(\frac{1}{2} - \frac{L'}{2}\alpha_t\right) \| \nabla F(x_t)\|^2 
    + \frac{\alpha_t}{2} \| \nabla F(x_t) - \EE_{\xi_t}\nabla \hat F^{(t)}_\biased \|^2 \\
    &\quad \quad \quad \quad + \frac{L'\alpha_t^2}{2}
    \left\{
     \|2\nabla F(x_t)\| \cdot \|\nabla F(x_t) - \EE_{\xi_t}\nabla \hat F^{(t)}_\biased\| + 
    \|\nabla F(x_t) - \EE_{\xi_t}\nabla \hat F^{(t)}_\biased\|^2
    + M_{\nabla \hat F_\biased} \right\}
    \\
    &\stackrel{2ab \leq a^2 + b^2}{\leq} F(x_t) - \alpha_t \left(\frac{1}{2} - \frac{L'}{2}\alpha_t\right) \| \nabla F(x_t)\|^2 
    + \frac{\alpha_t}{2} \| \nabla F(x_t) - \EE_{\xi_t}\nabla \hat F^{(t)}_\biased \|^2 \\
    &\quad \quad \quad \quad + \frac{L'\alpha_t^2}{2}
    \left\{
    \|\nabla F(x_t)\|^2 + 2 \|\nabla F(x_t) - \EE_{\xi_t}\nabla \hat F^{(t)}_\biased\|^2
    + M_{\nabla \hat F_\biased} \right\}
    \\
    &= F(x_t) -\alpha_t\left(\frac{1}{2} - L'\alpha_t\right) \| \nabla F(x_t)\|^2 
    + \alpha_t \left(\frac{1}{2} + L'\alpha_t\right)
    \| \nabla F(x_t) - \EE_{\xi_t}\nabla \hat F^{(t)}_\biased \|^2
    + \frac{L'M_{\nabla \hat F_\biased}\alpha_t^2}{2}
    \\
    &\stackrel{\text{Jensen's ineq.}}{\leq} F(x_t) -\alpha_t\left(\frac{1}{2} - L'\alpha_t\right) \| \nabla F(x_t)\|^2 
    + \alpha_t \left(\frac{1}{2} + L'\alpha_t\right)
    \| \nabla F(x_t) - \nabla \hat F^{(t)}_\biased \|^2
    + \frac{L'M_{\nabla \hat F_\biased}\alpha_t^2}{2}.
\end{align*}
By taking the summation of the above inequalities from $t=0, \ldots, T-1$, we get the main statement.
\end{proof}

To utilize the above lemma, we derive the upper bound on the bias term in the following.

\begin{lemma}[Bounded Bias in Delayed MLMC Gradient 1]\label{lemma:bounded_error_1}
Under Assumption \ref{assum:variance} and \ref{assum:smoothness}, suppose we run SGD with delayed MLMC gradient estimator $\nabla \hat F^{(t)}_\DMLMC$ as in Algorithm \ref{algo:delayed_mlmc}.
Then, we have 
\begin{align*}
\EE\|\nabla F(x_t) - \nabla \hat F^{(t)}_\DMLMC \|^2
&\leq 
\left\{
(\ellmax + 1) 
+ 2 \cdot \left(
    \sum_{\ell=0}^{\ellmax} 2^{-d\ell} L \sum_{r=\tau_\ell(t)}^{t-1} \alpha_r
\right)
\right\}
\\
&\quad \quad 
\times \left\{
    M' 
    + \sum_{\ell=0}^{\ellmax} 
    2^{-d\ell}L \sum_{s=\tau_\ell(t)}^{t-1} \alpha_s 
    \left(
    \EE \| \nabla F(x_s) - \nabla \hat F^{(s)}_\DMLMC\|^2
    + \EE \| \nabla F(x_s)\|^2
    \right)
\right\}.
\end{align*}
\end{lemma}

%
%
\begin{proof}
We will first decompose the norm of the bias as
\begin{align*}
&\|\nabla F(x_t) - \nabla \hat F^{(t)}_\DMLMC \| 
\\
&\quad = \left\| \sum_{\ell=0}^{\ellmax} \left[ \nabla \Delta_\ell F(x_t) -  \nabla \Delta_\ell \hat F_\MLMC(x_{\tau_\ell(t)}, \bsxi_{\tau_\ell(t), \ell})\right]\right\|
\\
&\quad \leq \sum_{\ell=0}^{\ellmax} \left\| \nabla \Delta_\ell F(x_t) -  \nabla \Delta_\ell \hat F_\MLMC(x_{\tau_\ell(t)}, \bsxi_{\tau_\ell(t), \ell})\right\|
\\
&\quad \leq \sum_{\ell=0}^{\ellmax} \left\{
\left\| \nabla \Delta_\ell F(x_t) - \nabla \Delta_\ell F(x_{\tau_\ell(t)}) \right\|
+ \left\| \nabla \Delta_\ell F(x_{\tau_\ell(t)}) -  \nabla \Delta_\ell \hat F_\MLMC(x_{\tau_\ell(t)}, \bsxi_{\tau_\ell(t), \ell})\right\| 
\right\}.
\end{align*}
By Assumption \ref{assum:smoothness}, we can bound the second term in the summation as
\begin{align*}
    \left\| \nabla \Delta_\ell F(x_t) - \nabla \Delta_\ell F(x_{\tau_\ell(t)}) \right\|
    \leq 2^{-d\ell}L\|x_t - x_{\tau_\ell(t)}\|
\end{align*}
and we can use the triangular inequality as
\begin{align*}
    \|x_t - x_{\tau_\ell(t)}\|
    &\leq \sum_{s=\tau_\ell(t)}^{t-1} \|x_{s+1} - x_s\| \\
    &\leq \sum_{s=\tau_\ell(t)}^{t-1} \|- \alpha_s \nabla \hat F^{(s)}_\DMLMC\| \\
    &\leq \sum_{s=\tau_\ell(t)}^{t-1} \alpha_s \left(
    \| \nabla F(x_s) - \nabla \hat F^{(s)}_\DMLMC\| + \| \nabla F(x_s)\| 
    \right).
\end{align*}

Here, we use a modified version of Cauchy-Schwarz inequality $(\sum_i u_i v_i)^2 \leq (\sum_i u_i)^2 (\sum_i v_i^2)$.
For non-negative $a_i$ and any $x_i$, we can show that
$$\left(\sum_i a_i x_i\right)^2 \leq \left( \sum_i a_i \right) \left( \sum_i a_i x_i^2 \right)$$
holds, by substituting $u_i=\sqrt{a_i}$ and $v_i=\sqrt{a_i}x_i$ to above.
Letting $x_i$'s and $a_i$'s be the norms and their coefficients, we get
\begin{align*}
&\EE\|\nabla F(x_t) - \nabla \hat F^{(t)}_\DMLMC \|^2 \\
&\leq \EE \Biggl| \sum_{\ell=0}^{\ellmax} \Biggl\{
\left\| \nabla \Delta_\ell F(x_{\tau_\ell(t)}) -  \nabla \Delta_\ell \hat F_\MLMC(x_{\tau_\ell(t)}, \bsxi_{\tau_\ell(t), \ell})\right\| \\
&\quad \quad + 2^{-d\ell}L \sum_{s=\tau_\ell(t)}^{t-1} \alpha_s \left(
    \| \nabla F(x_s) - \nabla \hat F^{(s)}_\DMLMC\| + \| \nabla F(x_s)\| 
    \right)
\Biggr\}
\Biggr|^2
\\
&\leq \left[
    \sum_{\ell=0}^{\ellmax} \left\{
        1 + 2^{-d\ell}L \sum_{s=\tau_\ell(t)}^{t-1} \alpha_s (1+1)
    \right\}
\right]
\\ & \quad \quad \times
\EE \Biggl[ \sum_{\ell=0}^{\ellmax} \Biggl\{
\left\| \nabla \Delta_\ell F(x_{\tau_\ell(t)}) -  \nabla \Delta_\ell \hat F_\MLMC(x_{\tau_\ell(t)}, \bsxi_{\tau_\ell(t), \ell})\right\| ^2 \\
&\quad \quad \quad \quad + 2^{-d\ell}L \sum_{s=\tau_\ell(t)}^{t-1} \alpha_s \left(
        \| \nabla F(x_s) - \nabla \hat F^{(s)}_\DMLMC\|^2
        + \| \nabla F(x_s)\|^2 
    \right)
\Biggr\}
\Biggr]
\\
&\leq \left\{
(\ellmax + 1) 
+ 2 \cdot \left(
    \sum_{\ell=0}^{\ellmax} 2^{-d\ell} L \sum_{s=\tau_\ell(t)}^{t-1} \alpha_s
\right)
\right\}
\\
&\quad \quad 
\times \left\{
    M' 
    + \sum_{\ell=0}^{\ellmax} 
    2^{-d\ell}L \sum_{s=\tau_\ell(t)}^{t-1} \alpha_s 
    \left(
    \EE \| \nabla F(x_s) - \nabla \hat F^{(s)}_\DMLMC\|^2
    + \EE \| \nabla F(x_s)\|^2
    \right)
\right\},
\end{align*}
which concludes the proof.
\end{proof}

Now, we derive a recursive formula of the upper bound from this lemma and obtain a more concrete bound on the bias.

\begin{lemma}[Bounded Bias in Delayed MLMC Gradient 2]\label{lemma:bounded_error_2}
Under Assumption \ref{assum:variance} and \ref{assum:smoothness}, suppose we run SGD with delayed MLMC gradient estimator $\nabla \hat F^{(t)}_\DMLMC$ as in Algorithm \ref{algo:delayed_mlmc}.
Additionally, assume that step sizes are chosen as $\alpha_t < \frac{\beta}{L}$ for $\beta$ satisfying 
$0 < \beta \leq \frac{1}{4(\ellmax + 1)\left(\sum_{\ell=0}^\infty2^{-d\ell}\right)\log(2T+1)} \leq \frac{1}{2}$.\footnote{Strictly speaking, we need to assume $T\geq 2$ for the third inequality.}
Then, there exist $K_1, K_2 > 0$ such that
\begin{align}
\EE\|\nabla F(x_t) - \nabla \hat F^{(t)}_\DMLMC \|^2
\leq K_1 M' + K_2\sum_{s=0}^{t - 1} \frac{1}{t - s} \EE \|\nabla F(x_s)\|^2
\label{eq:delayed_gradient_bias_bound}
\end{align}
for $t=0, \ldots, T-1$.
\end{lemma}

\begin{proof}
We first re-write the inequality in Lemma \ref{lemma:bounded_error_1} into a recursion as follows:
\begin{align}
&\EE\|\nabla F(x_t) - \nabla \hat F^{(t)}_\DMLMC \|^2 \nonumber
\\
&\leq 
\left\{
(\ellmax + 1) 
+ 2 \cdot \left(
    \sum_{\ell=0}^{\ellmax} 2^{-d\ell} L \sum_{r=\tau_\ell(t)}^{t-1} \alpha_r
\right)
\right\} \nonumber
\\
&\quad \quad 
\times \left\{
    M' 
    + \sum_{\ell=0}^{\ellmax} 
    2^{-d\ell}L \sum_{s=\tau_\ell(t)}^{t-1} \alpha_s 
    \left(
    \EE \| \nabla F(x_s) - \nabla \hat F^{(s)}_\DMLMC\|^2
    + \EE \| \nabla F(x_s)\|^2
    \right)
\right\} \nonumber 
\\
&\leq 
(2\beta + 1) (\ellmax + 1) \cdot \left\{
    M' 
    + \beta \cdot \sum_{\ell=0}^{\ellmax} 
    2^{-d\ell} \sum_{s=\tau_\ell(t)}^{t-1} 
    \left(
    \EE \| \nabla F(x_s) - \nabla \hat F^{(s)}_\DMLMC\|^2
    + \EE \| \nabla F(x_s)\|^2
    \right)
\right\} \nonumber
\\
&=
(2\beta + 1) (\ellmax + 1) M' + 
\beta (2\beta + 1) (\ellmax + 1) 
\sum_{\ell=0}^{\ellmax} \sum_{s=t - \lfloor2^{d\ell}\rfloor}^{t-1} 
    2^{-d\ell} 
    \left(
    \EE \| \nabla F(x_s) - \nabla \hat F^{(s)}_\DMLMC\|^2
    + \EE \| \nabla F(x_s)\|^2
    \right) \nonumber
\\
&\stackrel{\text{re-index}}{\leq}
(2\beta + 1) (\ellmax + 1) M' \nonumber
\\
&\quad\quad\quad\quad\quad + \beta (2\beta + 1) (\ellmax + 1) 
\sum_{s=0}^{t-1} 
\sum_{\ell=\left\lceil \frac{\log_2(t-s))}{d}\right\rceil}^{\infty} 
    2^{-d\ell} 
    \left(
    \EE \| \nabla F(x_s) - \nabla \hat F^{(s)}_\DMLMC\|^2
    + \EE \| \nabla F(x_s)\|^2
    \right) \nonumber
\\
&=
(2\beta + 1) (\ellmax + 1) M' \nonumber
\\
&\quad \quad \quad + 
\beta (2\beta + 1) (\ellmax + 1) 
\sum_{s=0}^{t-1} 
2^{-d \left\lceil \frac{\log_2(t-s))}{d} \right\rceil}
\sum_{\ell=0}^{\infty} 
    2^{-d\ell} 
    \left(
    \EE \| \nabla F(x_s) - \nabla \hat F^{(s)}_\DMLMC\|^2
    + \EE \| \nabla F(x_s)\|^2
    \right) \nonumber
\\
&\leq
(2\beta + 1) (\ellmax + 1) M' \nonumber
\\
&\quad \quad \quad + 
\beta (2\beta + 1) (\ellmax + 1) 
\sum_{s=0}^{t-1} 
2^{-d \frac{\log_2(t-s))}{d}}
\left( \sum_{\ell=0}^{\infty}  2^{-d\ell}  \right)
    \left(
    \EE \| \nabla F(x_s) - \nabla \hat F^{(s)}_\DMLMC\|^2
    + \EE \| \nabla F(x_s)\|^2
    \right) \nonumber
\\
&=
C_1 M' + 
C_2 
\sum_{s=0}^{t-1} 
\frac{1}{t-s}
    \left(
    \EE \| \nabla F(x_s) - \nabla \hat F^{(s)}_\DMLMC\|^2
    + \EE \| \nabla F(x_s)\|^2
    \right). \label{eq:rewrite_delayed_gradient_bias_bound}
\end{align}
Here, we introduced constants $C_1 := (2\beta + 1) (\ellmax + 1)$ and $C_2:= \beta (2\beta + 1) (\ellmax + 1) \left( \sum_{\ell=0}^{\infty}  2^{-d\ell}  \right)$ at the last line.
In the above reformulation, we re-indexed the summation using the fact that the summand is always non-negative and that 
\begin{align*}
&\{(\ell, s): 0\leq \ell \leq \ellmax, 0\leq s\text{ and } t-\lfloor 2^{d\ell}\rfloor \leq s \leq t - 1\}
\\
&= \{(\ell, s): 0\leq \ell \leq \ellmax, t-s \leq \lfloor 2^{d\ell}\rfloor \text{ and } 0\leq s \leq t - 1\}
\\
&\subseteq \{(\ell, s): 0\leq \ell \leq \ellmax, t-s \leq 2^{d\ell}\text{ and } 0 \leq s \leq t - 1\}
\\
&\subset \{(\ell, s): (\log_2(t-s))/d \leq \ell \text{ and } 0 \leq s \leq t - 1\}.
\end{align*}
For notational simplicity, we used a convention of taking summation with respect to only non-negative indices in the above.

Now, we look for $K_1, K_2 \geq 0$ for which we can use the mathematical induction from $t=0$ to $t=T-1$ to prove \eqref{eq:delayed_gradient_bias_bound}.
When $t=0$, the left-hand side of the inequality simply becomes the variance of the standard MLMC estimator as the coupled gradient estimators for all levels are calculated, making the estimator unbiased.
Thus, \eqref{eq:delayed_gradient_bias_bound} holds for any $K_1\geq 1$ and $K_2>0$ at $t=0$.
Next, for the mathematical induction, let us assume that \eqref{eq:delayed_gradient_bias_bound} holds for any $0\leq t \leq t' - 1$. Then, by \eqref{eq:rewrite_delayed_gradient_bias_bound}, we get
\begin{align*}
&\EE\|\nabla F(x_t) - \nabla \hat F^{(t)}_\DMLMC \|^2 \nonumber
\\
&\leq
C_1 M'
+ C_2 \sum_{s=0}^{t-1} 
\frac{1}{t-s}
    \left(
    \EE \| \nabla F(x_s) - \nabla \hat F^{(s)}_\DMLMC\|^2
    + \EE \| \nabla F(x_s)\|^2
    \right)
\\
&\leq
C_1 M'
+ C_2 \sum_{s=0}^{t-1} 
\frac{1}{t-s}
    \left(
        K_1 M'
        + K_2\sum_{u=0}^{s - 1} \frac{1}{s - u} \EE \|\nabla F(x_u)\|^2
        + \EE \| \nabla F(x_s)\|^2
    \right)
\\
&=
\left\{ C_1 + C_2 K_1 \left(\sum_{s=0}^{t-1} \frac{1}{t-s}\right) \right\} M'
+ C_2 K_2\sum_{s=0}^{t-1} \sum_{u=0}^{s - 1} \frac{1}{(t-s)(s - u)} \EE \|\nabla F(x_u)\|^2
+ C_2 \sum_{s=0}^{t-1}  \frac{1}{t-s} \EE \| \nabla F(x_s)\|^2
\\
&=
\left\{ C_1 + C_2 K_1 \left(\sum_{s=0}^{t-1} \frac{1}{t-s}\right) \right\} M'
+ C_2 K_2\sum_{u=0}^{t - 2} \left( \sum_{s=u+1}^{t-1} \frac{1}{(t-s)(s - u)} \right) \EE \|\nabla F(x_u)\|^2 \\
&\quad \quad + C_2 \sum_{s=0}^{t-1}  \frac{1}{t-s} \EE \| \nabla F(x_s)\|^2.
\end{align*}
To further the analysis, we need to bound the summation terms $\sum_{s=0}^{t-1}\frac{1}{t-s}$ and $\sum_{s=u+1}^{t-1} \frac{1}{(t-s)(s - u)}$. These terms can be bounded using the convexity of $x\mapsto \frac{1}{x}$ on domain $0<x$ and $x\mapsto \frac{1}{x(t-x)}$ on $0 < x < t$ as
\begin{align}
    \sum_{s=0}^{t-1}\frac{1}{t-s}
    &= \sum_{s=0}^{t-1}\frac{1}{\int_{t-s-\frac{1}{2}}^{t-s+\frac{1}{2}}x\mathrm{d}x} \nonumber
    \\
    &\stackrel{\text{Jensen's ineq.}}{\leq} \sum_{s=0}^{t-1} \int_{t-s-\frac{1}{2}}^{t-s+\frac{1}{2}}\frac{1}{x}\mathrm{d}x \nonumber
    \\
    &= \int_{\frac{1}{2}}^{t+\frac{1}{2}}\frac{1}{x}\mathrm{d}x  \nonumber
    \\
    &= \log(2t+1) \nonumber 
\end{align}
and
\begin{align*}
    \sum_{s=u+1}^{t-1} \frac{1}{(t-s)(s - u)}
    &= \sum_{k=1}^{t-u-1}\frac{1}{k(t-u-k)}
    \\
    &= \sum_{k=1}^{t-u-1}\frac{1}{
            \left(\int_{k-\frac{1}{2}}^{k+\frac{1}{2}}x\mathrm{d}x\right)
            \left(t-u-\int_{k-\frac{1}{2}}^{k+\frac{1}{2}}x\mathrm{d}x\right)
        }
    \\
    &\stackrel{\text{Jensen's ineq.}}{\leq}
        \sum_{k=1}^{t-u-1}\int_{k-\frac{1}{2}}^{k+\frac{1}{2}}\frac{\mathrm{d}x}{x(t-u-x)}
    \\
    &= \int_{\frac{1}{2}}^{t-u-\frac{1}{2}}\frac{\mathrm{d}x}{x(t-u-x)}
    \\
    &= \frac{1}{t-u} \int_{\frac{1}{2(t-u)}}^{1-\frac{1}{2(t-u)}}\frac{\mathrm{d}z}{z(1-z)}
    \\
    &= \frac{1}{t-u} \int_{\frac{1}{2(t-u)}}^{1-\frac{1}{2(t-u)}}\left(\frac{1}{z} + \frac{1}{1-z}\right)\mathrm{d}z
    \\
    &= \frac{1}{t-u} \cdot 2\log(2t-2u-1)
\end{align*}

Using these bounds, we get
\begin{align*}
&\EE\|\nabla F(x_t) - \nabla \hat F^{(t)}_\DMLMC \|^2 \nonumber
\\
&\leq
\left\{ C_1 + C_2 K_1 \log(2t+1) \right\} M'
+ C_2 K_2 \sum_{u=0}^{t - 2} \left( \frac{1}{t-u} \cdot 2\log(2t-2u-1) \right) \EE \|\nabla F(x_u)\|^2 \\
&\quad\quad + C_2 \sum_{s=0}^{t-1} \frac{1}{t-s} \EE \| \nabla F(x_s)\|^2.
\\
&\leq
\left\{ C_1 + C_2 K_1 \log(2t+1) \right\} M'
+ C_2 \left\{1 + K_2 2\log(2t-1) \right\} \sum_{s=0}^{t-1} \frac{1}{t-s} \EE \| \nabla F(x_s)\|^2.
\\
&\leq
\left\{ C_1 + C_2 K_1 \log(2T+1) \right\} M'
+ C_2 \left\{1 + K_2 2\log(2T-1) \right\} \sum_{s=0}^{t-1} \frac{1}{t-s} \EE \| \nabla F(x_s)\|^2.
\end{align*}
Therefore, by choosing 
\begin{alignat*}{3}
&K_1 &&= \frac{C_1}{1 - C_2 \log (2T + 1)} 
&&= \frac{(2\beta+1)(\ellmax+1)}{1 - \beta(2\beta+1)(\ellmax+1)\left(\sum_{\ell=0}^\infty 2^{-d\ell}\right)\log(2T+1)}
\\
\text{and}
\\
&K_2 &&= \frac{C_2}{1 - 2C_2\log (2T - 1)}
&&= \frac{(2\beta+1)(\ellmax+1)}{1 - \beta(2\beta+1)(\ellmax+1)\left(\sum_{\ell=0}^\infty 2^{-d\ell}\right)\log(2T-1)}
\end{alignat*}
so that $\left\{ C_1 + C_2 K_1 \log(2t+1) \right\}\leq K_1$ and $C_2 \left\{1 + K_2 2\log(2t-1) \right\} \leq K_2$, we can apply mathematical induction to prove \eqref{eq:delayed_gradient_bias_bound}.
\end{proof}

\begin{remark}\label{rm:K1K2}
In the last part of the above proof, due to the upper bound of $\beta$ in the assumption, the denominators of $K_1$ and $K_2$ are positive because
\begin{align*}
    1-C_2\log(2T-1) 
    &> 1-C_2\log(2T+1) 
    \\
    &\geq 1 - \beta(2\beta+1)(\ellmax + 1)\left(\sum_{\ell=0}^\infty 2^{-d\ell}\right)\log(2T+1)
    \\
    &\geq 1 - \beta(2\cdot \frac{1}{2}+1)(\ellmax + 1)\left(\sum_{\ell=0}^\infty 2^{-d\ell}\right)\log(2T+1)
    \\
    &\geq 1 - \frac{1}{4(\ellmax + 1)\left(\sum_{\ell=0}^\infty 2^{-d\ell}\right)\log(2T+1)}
        \cdot 2(\ellmax + 1)\left(\sum_{\ell=0}^\infty 2^{-d\ell}\right)\log(2T+1)
    \\
    &=\frac{1}{2}
\end{align*}
This implies that we can upper bound $K_1$ and $K_2$ as
\begin{alignat*}{3}
    &K_1 
    &&\leq 2(2\beta+1)(\ellmax + 1)
    &&\leq 4(\ellmax + 1)
    \\
    \text{and}
    \\
    &K_2 
    &&\leq 2\beta (2\beta + 1) (\ellmax + 1) \left( \sum_{\ell=0}^{\infty}  2^{-d\ell}  \right)
    &&\leq 4\beta (\ellmax + 1) \left( \sum_{\ell=0}^{\infty}  2^{-d\ell}  \right).
\end{alignat*}
so that the upper bounds don't depend on $T$.
\end{remark}

Finally, with the above bound on the bias, we obtain the main theorem:

\begingroup
\def\thetheorem{\ref{thm:delayed_mlmc}}
\begin{theorem}[Delayed MLMC Gradient Descent for Non-Convex Functions]
Under Assumption \ref{assum:variance} and \ref{assum:smoothness}, suppose we run SGD with delayed MLMC gradient estimator $\nabla \hat F^{(t)}_\DMLMC$ as in Algorithm \ref{algo:delayed_mlmc}.
Assume that the step sizes are chosen as $\alpha_t=\alpha_0\leq\min\left\{\frac{1}{8L'}, \frac{\beta}{L}\right\}$ for $\beta$ satisfying $0 < \beta \leq \frac{1}{12(\ellmax + 1)\left(\sum_{\ell=0}^{\infty}2^{-d\ell}\right)\log(2T+1)}$. Then, we have
\begin{equation*}
    \frac{1}{T}\sum_{t=0}^{T-1}\EE\|\nabla F(x_t)\|^2
    \leq \frac{8(F(x_0) - F_\textinf)}{\alpha_0 T}
    + \left( 24\ellmax + \frac{49}{2} \right) M'.
    \leq \Ocal\left(\left(\frac{\log T}{T} + \frac{M}{N} \right)\ellmax\right).
\end{equation*}
\end{theorem}
\addtocounter{theorem}{-1}
\endgroup

\begin{proof}
By Lemma \ref{th:biased_sgd} and Lemma \ref{lemma:bounded_error_2}, we have 
\begin{align*}
    &\sum_{t=0}^{T-1}\alpha_0\left(\frac{1}{2} - L'\alpha_0\right)\EE\|\nabla F(x_t)\|^2 \\
    &\leq F(x_0) - F_\textinf + \sum_{t=0}^{T-1} \left[
    \alpha_0 \left(\frac{1}{2} + L'\alpha_0\right) \EE \| \nabla F(x_t) 
    - \nabla \hat F^{(t)}_\DMLMC \|^2
    + \frac{L'M'\alpha_0^2}{2}
    \right]
    \\
    &\leq F(x_0) - F_\textinf + \sum_{t=0}^{T-1} \left[
    \alpha_0 \left(\frac{1}{2} + \frac{1}{4}\right) \left(
        K_1 M' + K_2 \sum_{s=0}^{t-1}\frac{1}{t-s}\EE\|\nabla F(x_s)\|^2
    \right)
    + \frac{L'M'\alpha_0^2}{2}
    \right]
    \\
    &\stackrel{\text{Remark \ref{rm:K1K2}}}{\leq} F(x_0) - F_\textinf
    + \sum_{t=0}^{T-1} \left[
        \alpha_0 \cdot \frac{3}{4} \cdot 4(\ellmax + 1) \cdot  M' + \frac{L'M'\alpha_0^2}{2}
    \right]
    \\
    &\quad \quad \quad + \sum_{t=0}^{T-1} \left[
        \alpha_0 \cdot \frac{3}{4} \cdot 4 \beta (\ellmax + 1) 
             \left(\sum_{\ell=0}^\infty 2^{-d\ell}\right) \left( \sum_{s=0}^{t-1}\frac{1}{t-s}\EE\|\nabla F(x_s)\|^2\right)
    \right]
    \\
    &= F(x_0) - F_\textinf
    + \sum_{t=0}^{T-1} \alpha_0 \left[ 3 (\ellmax + 1) + \frac{L'}{2}\alpha_0 \right] M'
    \\
    &\quad \quad \quad + 3\beta(\ellmax + 1) \cdot \alpha_0 \left(\sum_{\ell=0}^\infty 2^{-d\ell}\right)
    \sum_{s=0}^{T-2} \left( \sum_{t=s+1}^{T-1} \frac{1}{t-s} \right) \left( \EE\|\nabla F(x_s)\|^2\right).
\end{align*}
Here, summation $\sum_{t=s+1}^{T-1}\frac{1}{t-s}$ can be bounded by Jensen's inequality as
\begin{align*}
    \sum_{t=s+1}^{T-1}\frac{1}{t-s}
    &= \sum_{t=1}^{T-s-1}\frac{1}{t}
    \\
    &= \sum_{t=1}^{T-s-1}\frac{1}{\int_{t-\frac{1}{2}}^{t+\frac{1}{2}}x\mathrm{d}x}
    \\
    &\stackrel{\text{Jensen's ineq.}}{\leq} \sum_{t=1}^{T-s-1}\int_{t-\frac{1}{2}}^{t+\frac{1}{2}}\frac{1}{x}\mathrm{d}x
    \\
    &= \int_{\frac{1}{2}}^{T-s+\frac{1}{2}}\frac{1}{x}\mathrm{d}x
    \\
    &= \log(2T-2s+1).
\end{align*}
Thus, we get
\begin{align*}
    &\sum_{t=0}^{T-1}\alpha_0\left(\frac{1}{2} - L'\alpha_0\right)\EE\|\nabla F(x_t)\|^2
    \\
    &= F(x_0) - F_\textinf
    + \sum_{t=0}^{T-1} \alpha_0 \left[ 3 (\ellmax + 1) + \frac{L'}{2}\alpha_0 \right] M'
    \\
    &\quad \quad \quad + 3\beta(\ellmax + 1) \cdot \alpha_0  \left(\sum_{\ell=0}^\infty 2^{-d\ell}\right)
    \sum_{s=0}^{T-2} \log(2T - 2s + 1) \EE\|\nabla F(x_s)\|^2
    \\
    &\leq F(x_0) - F_\textinf
    + \sum_{t=0}^{T-1} \alpha_0 \left[ 3 (\ellmax + 1) + \frac{L'}{2}\alpha_0 \right] M'
    \\
    &\quad \quad \quad + 3\beta(\ellmax + 1) \cdot \alpha_0  \left(\sum_{\ell=0}^\infty 2^{-d\ell}\right) \log(2T + 1)
    \sum_{s=0}^{T-1} \EE\|\nabla F(x_s)\|^2
    \\
    &\leq F(x_0) - F_\textinf
    + \sum_{t=0}^{T-1} \alpha_0 \left[ 3 (\ellmax + 1) + \frac{L'}{2}\alpha_0 \right] M'
    \\
    &\quad \quad \quad + 
    \frac{1}{12(\ellmax + 1)\left(\sum_{\ell=0}^\infty 2^{-d\ell}\right)\log(2T+1)} 
    \cdot 3(\ellmax + 1) \cdot \alpha_0  \left(\sum_{\ell=0}^\infty 2^{-d\ell}\right) \log(2T + 1)
    \sum_{s=0}^{T-1} \EE\|\nabla F(x_s)\|^2
    \\
    &\leq F(x_0) - F_\textinf
    + \sum_{t=0}^{T-1} \alpha_0 \left[ 3 (\ellmax + 1) + \frac{L'}{2}\alpha_0 \right] M'
    + \frac{1}{4} \alpha_0 \sum_{s=0}^{T-1} \EE\|\nabla F(x_s)\|^2.
\end{align*}
Therefore, we have
\begin{align*}
    \sum_{t=0}^{T-1} \alpha_0 \cdot \frac{1}{8} \cdot \EE\|\nabla F(x_t)\|^2
    &\leq \sum_{t=0}^{T-1}\alpha_0\left(\frac{1}{4} - L'\alpha_0\right)\EE\|\nabla F(x_t)\|^2
    \\
    &\leq F(x_0) - F_\textinf
    + \sum_{t=0}^{T-1} \alpha_0 \left[ 3 (\ellmax + 1) + \frac{L'}{2}\alpha_0 \right] M'.
    \\
    &\leq F(x_0) - F_\textinf
    + \sum_{t=0}^{T-1} \alpha_0 \left[ 3 (\ellmax + 1) + \frac{L'}{2} \cdot \frac{1}{8L'} \right] M'.
\end{align*}
from which the first inequality follows.

The second inequality follows trivially by substituting the upper bound on $\alpha_0$ to the first upper bound.
\end{proof}

\section{Experimental Settings}\label{app:experimental_settings}

In our numerical experiment, we employed deep hedging \citep{buehler2019deep} as an example.
The deep hedging involves minimization of the following objective \cite[Equation 3.3]{buehler2019deep}:
\begin{equation*}
   \min_{\theta\in\Theta, p_0\in\RR}\EE\left|
       \max\{S_1 - K, 0\} - \int_0^1 H_\theta(t, S_t)\rd S_t - p_0
\right|^2.
\end{equation*}
This optimization problem aims to determine optimal hedging strategy $H_\theta(t, s)$ and initial price $p_0$ of European call option with maturity at $t=1$, which is $p_0 = \EE\left[\max\{S_1 - K, 0\}\right]$.
Hedging strategy $H_\theta(t, s)$ represents the amount of the underlying asset we hold to hedge against a share of sold European call option with strike price $K$, whose payoff can be written as $\max\{S_1 - K, 0\}$.
For price process of the underlying asset $\{S_t\}_{t\in[0, 1]}$, we choose geometric Brownian motion model with drift $\mu$ and volatility $\sigma$, which follows
\begin{equation*}
    \rd S_t = \mu \rd t + \sigma S_t \rd B_t
\end{equation*}
for standard Brownian motion $\{B_t\}_{t\in[0, 1]}$.
To solve the SDE, we employed the Milstein scheme, a standard solver for MLMC simulation of SDEs \citep{giles2008improved}.
Hedging model $H_\theta(t, s)$ was implemented as a feed-forward neural network with 2 hidden layers, each comprising 32 nodes.
We used the SiLU activation \citep{elfwing2018sigmoid} for all layers except the final layer, for which we used the sigmoid activation.
This choice of activation functions ensures that the objective function is smooth and that the holding volume of the hedging strategy is within the valid range of $[0, 1]$. 
For solving the resulting neural SDE, we used Diffrax \cite{kidger2021on}, a library for neural differential equations based on Jax \citep{jax2018github}.
The parameter values for the simulation were set as follows:
$c=1$,
$d=1$,
$b=1.8$
$\ellmax=6$,
$\mu=1$,
$\sigma=1$, and
$K=3$.

\end{document}